\documentclass[letterpaper]{article}
\usepackage{bm}
\usepackage{graphicx}
\usepackage{subfigure}
\usepackage{caption}
\usepackage{amsmath,amsthm,amssymb}
\usepackage[ruled,vlined]{algorithm2e}
\usepackage{grffile}
\newcommand{\Matrix}[1]{\bm{\mathbf{#1}}}
\usepackage{indentfirst}
\usepackage{amsmath}
\usepackage{url}
\usepackage{hyperref}
\newtheorem{theorem}{Theorem}
\newtheorem{lemma}{Lemma}

\begin{document}

\title{Reinforcement Learning for Uplift Modeling}


\author{Chenchen Li$^{1,2}$ \thanks{Equal contribution}, Xiang Yan$^{1,2}$ \footnotemark[1], Xiaotie Deng$^3$,  Yuan Qi$^1$, Wei Chu$^1$,\\Le Song$^{1}$,
Junlong Qiao$^1$, Jianshan He$^1$, Junwu Xiong$^1$\\ $^1$AI Department, Ant Financial Services Group\\ $^2$Department of Computer Science, Shanghai Jiao Tong University\\ $^3$School of Electronics Engineering and Computer Science, Peking University\\ lcc1992@sjtu.edu.cn, xyansjtu@163.com, xiaotie@pku.edu.cn\\
 \{yuan.qi, weichu.cw, le.song, junlong.qjl, yebai.hjs, junwu.xjw\}@antfin.com}

\maketitle

\begin{abstract}
Uplift modeling aims to directly model the incremental impact of a treatment on an individual response.
In this work, we address the problem from a new angle and reformulate it as a Markov Decision Process (MDP).
We conducted extensive experiments on both a synthetic dataset and real-world scenarios, and showed that our method can achieve significant improvement over previous methods.
\end{abstract}

\section{Introduction}

Uplift modeling aims to directly model the incremental impact of a treatment/action on an individual response.
In contrast to traditional classification techniques where one focused on directly predicting the response, uplift modeling focuses on estimating the net effect of a specific treatment by modeling the difference between one's response before and after a particular treatment.
A typical example appears in modern marketing.
After a restaurant sends out coupons to passersby, some of them come to eat.
Part of these customers are attracted by the coupons but the others may have already planned to eat there before receiving the coupon.
Meanwhile among those who do not come, some are not interested in coupons but some may be annoyed to change their minds.
In fact, what really matters to the restaurant is the difference between behaviors of the same person before and after he receives a coupon.
Uplift modeling is also important in many other settings such as personalized recommendation \cite{rzepakowski2012uplift}, medical treatments \cite{jaskowski2012uplift}, causal inference \cite{gutierrez2017causal} and so on.

In this paper, we refer to a customer's observable response after we take a specific action on him the \textbf{action response}, and the corresponding behavior when we take no action the \textbf{natural response}. Thus the main concern of an uplift modeling problem is the difference between \textbf{action response} and \textbf{nature response} of a customer, which is called \textbf{uplift response} with respect to a specific action.

Besides its wide applications, uplift modeling also receives much attention from the machine learning community. 
Traditional machine learning methods have been attempted to tackle the problem, such as K-Nearest Neighbors \cite{su2012facilitating}, Support Vector Machine \cite{zaniewicz2013support}, Decision Trees \cite{rzepakowski2010decision} and Random Forests \cite{guelman2015uplift}.
However, there still exist two unsolved difficulties in previous works, which restrict these methods' performance in applications.
On the one hand, an unbiased evaluation metric for uplift modeling in this case remains missing.
Some existing metrics, such as Qini coefficient and uplift curve \cite{radcliffe2011real}, are only suitable for the cases of a single action and binary response. 
Such a lack of evaluation metric makes it difficult to carry out analysis using offline datasets.
On the other hand, in uplift modeling, one typically does not know individual responses to an action and corresponding natural responses at the same time, which means the explicit labels of uplift response for specific features and actions are not available.
Currently, tree-based methods are widely used to handle the lack of explicit labels \cite{chickering2000decision,hansotia2002incremental,rzepakowski2010decision,radcliffe2011real,guelman2014survey,zhao2017uplift}, but these methods need manually engineered features which are less automated in comparison to deep learning methods \cite{lecun2015deep}.

To handle the above two limitations, we propose a new evaluation uplift modeling metric for any number of actions and general response types (binary, discrete, continuous), which is a variant of inverse propensity score (IPS) for uplift modeling.
We prove that it is an unbiased estimation of the uplift response.
Then we reformulate the uplift modeling problem as a Markov Decision Process (MDP) and adopt the neuralized policy gradient method \cite{sutton2000policy} to solve the problem.
Such a deep reinforcement learning approach can automatically learn representations from the data, and requires no explicit label of each sample like supervised learning. 
It uses only a positive or negative reward to guide which action is good in the specific environment.
And we further adopt the action-dependent baseline to reduce the variance of gradients of reinforcement learning, which has been shown to be effective in the recent works \cite{tucker2018mirage}.

In experiments, we first verify the efficiency of our proposed metric by showing its average convergence rate and variance in multi-fold synthetic datasets. 
Then, our method, \textbf{RLift}, is adopted in extensive experiments on an open dataset, a synthetic dataset and real-world scenarios. 
All results show that our method can achieve significant improvement over state-of-the-art methods according to the evaluation of our proposed metric as well as traditional metrics.

It is worth noticing that the offline contextual bandit problem \cite{beygelzimer2009offset} has quite a similar setting but a different objective compared with uplift modeling.
Both problems require a policy for deciding actions by taking advantages of offline data about individuals' observable action response. But an offline contextual bandit problem asks a policy which maximizes the expected \textbf{action response}, while uplift modeling asks one maximizing the expected \textbf{uplift response}.
With proper formalization, the optimal solution to an offline contextual bandit problem can be transformed to the optimal solution to a specific uplift modeling problem in closed form.
However, current methods for both problems are seeking an approximation of the optimal solution, resulting in an inherent difference between them.
We review methods mentioned in literatures for both problems in Section \ref{related_works}, and analyze the difference between them carefully in Section \ref{diff}.
In experiments, we also compare our method with one famous method solving the offline contextual bandit problem, Offset Tree\cite{beygelzimer2009offset}. 
In precise, the output policy by Offset Tree is interpreted as one for uplift modeling and its performance is not as good as our method, coincides with our analysis.

\subsection{Related Works}\label{related_works}

The most direct approach for uplift modeling is the Separate Model Approach (SMA), which uses separated model for each group of people receiving the same action, predicts the corresponding responses, then chooses the action with the maximum response \cite{radcliffe2011real}.
It can make use of any supervised learning methods and performs well when the uplift response is strongly correlated to action response. 
However, it performs badly when the uplift response follows a different distribution with the action response, which was illustrated in \cite{radcliffe2011real}.

On the other hand, the variations of the decision tree-based methods model the uplift response directly in order to avoid the weakness.
In the traditional decision tree, the algorithm chooses the attribute with maximum splitting criteria when growing each time, which aims to maximize the entropy after splitting \cite{chickering2000decision,hansotia2002incremental,rzepakowski2010decision,guelman2015uplift}.
For example, the criteria in \cite{hansotia2002incremental} is to maximize the difference of response signals between the child nodes, while the one in \cite{rzepakowski2010decision} is to maximize the distributional difference of response signals between child nodes by weighted Kullback-Leibler divergence and weighted squared Euclidean distance.
Previous state-of-art method \cite{zhao2017uplift} also uses random forests, but it predicts the exact value of response instead of the value of uplift.
They also proposed an unbiased evaluation metric for multiple treatments and general response type, but the metric is for the response of people after receiving treatments rather the difference between that of ones receiving treatments and receiving nothing.

Besides tree-based methods, the adaption of K-Nearest Neighbors (KNN) was considered for uplift modeling \cite{alemi2009improved，su2012facilitating}.
KNN was used to find objects with similar features, then the action response on the similar objects.
A logistic regression formulation is also proposed to explicitly include interaction terms between features and the action \cite{lo2002true}.
And Support Vector Machine (SVM) is considered to find hyperplanes in order to divide the feature space into positive, neutral and negative part affected by the action \cite{zaniewicz2013support}.
Due to the lack of performance metrics, these methods mainly work in the single action and binary response case, and did not achieve stable performance in practice.

On the contrast, reinforcement learning has been famous for a lot of successful applications in many fields, such as the Go game \cite{silver2017mastering}, video games \cite{mnih2015human}, and so on.
One of its advantages is that it needs no explicit labels, but the reward signals to guide the training. 
One of reinforcement learning methods, policy gradient \cite{sutton2000policy}, is suitable for the episode task and receives the delayed rewards because it can calculate the gradient after the entire episode, as Feng and Zhang did on the task of relation classification and learning structured representation for text classification \cite{feng2018reinforcement,zhang2018learning}.

Furthermore, variance reduction methods are believed to help the training process of a reinforcement learning \cite{greensmith2004variance}. For example, advantage functions \cite{sutton2011reinforcement} are commonly used to reduce variances when estimating value functions. 
Recent works \cite{tucker2018mirage,wu2018variance,liu2018action} show that the action-dependent advantage functions reduce the variance significantly, especially for discrete action space \cite{tucker2018mirage}, which is adopted in our policy gradient approach.

As for the evaluation of uplift modeling through offline datasets, Qini coefficients and uplift curves are widely used \cite{radcliffe2007using,rzepakowski2010decision} for single action cases and perform well in practice, despite their lack of theoretical justification.
For the multiple actions cases, \cite{zhao2017uplift} proposed a performance metric for expected action response, but not for the expected uplift response. 
In the field of reinforcement learning, IPS has been studied widely for offline policy evaluation \cite{swaminathan2015self,swaminathan2015counterfactual,dudik2011doubly}.
Unlike previous uplift metric, IPS does not require the training samples to be collected randomly.
However, there has not been a version of IPS for uplift modeling.

In addition, some works related to other topics are similar to uplift modeling, such as offline contextual bandit \cite{swaminathan2015self,swaminathan2015counterfactual,beygelzimer2009offset}.
Offline contextual bandit pays attention to find a policy to maximize the expected response with the offline logged dataset.
Based on the partial observable response, \cite{beygelzimer2009offset} transform the bandit problem into a cost-sensitive supervised learning problem.
And with the help of IPS, the policy can be optimized directly\cite{swaminathan2015counterfactual}, and a further self-normalized estimator helps reduce variance and avoid overfitting\cite{swaminathan2015self}.
These problems are very similar to ours, where the optimal solutions are the same, but the approximation solutions are not, which will be analyzed in Section \ref{diff}.

On the other hand, in the field of causal inference, the uplift response of an action with respect to each single features sample is defined as individual treatment effect (ITE)\cite{wager2017estimation,shalit2016estimating}. 
Causal inference community focuses on estimating such an effect of the single action for each specific feature vector accurately, while the metric we propose in this paper evaluates a policy choosing multiple actions by estimating the expected uplift response on whole features space.

\subsection{Organization of the Article}
In Section 2, we provide a formal definition and an unbiased metric for uplift modeling.
In Section 3, we present our deep reinforcement learning design for uplift modeling.
In Section 4, we 
compare our methods with several baselines on an open dataset, a group of simulation datasets and a real business dataset.

\section{Uplift Modeling and an Unbiased Evaluation Metric}\label{sec:uplift_model}
In this section, we first provide a formal definition of the uplift modeling problem, and then we propose a multi-action and general response type evaluation metric of uplift and prove its unbiasedness.
Finally, we analyze the specialty of uplift modeling comparing with other similar problems.

\subsection{Definition of Uplift Modeling for Multiple Actions and General Response Types}

Firstly, we introduce basic elements and their notations in uplift modeling. When multiple actions are taken for individuals, we have

\begin{itemize}
    \item $\Matrix{X}$: the variable of individual's feature vector. We use $x$ to denote its one realization. It usually represents the feature of one customer or one disease.
    \item $a\in \{ 0, 1, \dots, K \}$: the encoded action. Specifically $a=0$ means no action.
    \item $\pi(\cdot)$: a policy of choosing actions for each feature. We use $\pi(x)=a$ to denote a realization that the policy selects action $a$ for $x$, and $\pi(a|x)$ the corresponding probability for such a realization.
    \item $Y(x, a)$: the observed action response when $x$ receives action $a$. Generally speaking, the response is a real number.
    \item $B(x)$: the nature response of $x$ when receiving no action.
    \item $L(x, a)$: the uplift response when $x$ receives action $a$.
\end{itemize}

Now we can formally define the uplift modeling by the commonly used additive model\cite{radcliffe2011real}.
\begin{equation}
    Y(x, a) = B(x) + L(x, a)
\end{equation}, which means the observed action response contains two parts, the nature response and uplift response.
The nature response is independent on the action, while the uplift response depends on the action. Naturally, $Y(x,0) = B(x)$ and $L(x, 0) = 0$.

And the goal of uplift modeling is to find a policy $\pi(\cdot)$ to maximize the expected uplift response. Formally,
\begin{equation}
    \label{equ:def_uplift}
    \max_{\pi} \mathbf{E} _{\Matrix{X}} [L(\Matrix{X}, \pi(\Matrix{X}))].
\end{equation}

\subsection{Uplift Modeling General Metric}
Before we seek a policy for the uplift modeling problem, we need an unbiased metric of uplift response for any specific policy.
This is because the uplift response can never be observed directly.
Online experiments usually use A/B testing to estimate the uplift response, but the experiments may cost a lot.
So we consider an offline evaluation for policies, which takes advantage of a dataset from a previous experiment. 
In precise, the dataset has $N$ samples in total, containing $K$ groups with different actions and a control group without any actions.
Each individual with feature $x$ was once assigned into one of $K+1$ different groups by a policy $p$, for example in the previous experiment action $a$ was taken on $x$, and the corresponding response $Y(x,a)$ is recorded.
For notational convenience, the probability that policy chooses an action $a$ for feature $x$ is denoted by $p(a|x)$, where $\sum_{a=0}^{K} p(a|x) = 1$.
Note that unlike previous works on uplift modeling requiring the policy $p(\cdot)$ to be independent on feature $x$, we only require that for $p(a|x)$ is independent on $Y(x,a)$.

Based on the dataset, we design an unbiased estimation of the uplift response for a specific policy $\pi$.
It is done by estimating the expected action responses for the policy (Lemma 1) and the nature response (Lemma 2) separately.
Here we use $Pr(\cdot)$ to denote the probability of a random event, and $\mathbb{I}\{\cdot\}$ is the 0/1 indicator function. And $\pi(x)=a$ represents a realization that policy $\pi$ chooses action $a$ for $x$ while $p(x)=a$ represents the fact that action $a$ is taken on $x$ in the dataset.
\begin{lemma} 
    Given a policy $\pi$, for each action $a\in \{0,1,\dots,K\}$, define a new random variable
    $$ Z^{T}(a,\Matrix{X})=\frac{Y(\Matrix{X},a)\mathbb{I}\{p(\Matrix{X})=a\}}{p(a|\Matrix{X})} $$
    Then
    $$\mathbf{E}_{\Matrix{X}} [Z^{T}(a,\Matrix{X})|\pi(\Matrix{X})=a]=\mathbf{E}_{\Matrix{X}}[Y(\Matrix{X}, \pi(\Matrix{X}))|\pi(\Matrix{X})=a] $$
\end{lemma}
\begin{proof}
\begin{equation*}
\begin{aligned}
&\mathbf{E}_{\Matrix{X}}[Z^{T}(a,\Matrix{X})|\pi(\Matrix{X})=a]\\
=&\sum_{x\in\Matrix{X}}\frac{Y(x, a)|\mathbb{I}\{\pi(x)=a\}}{p(a|x)}\frac{\pi(a|x)p(a|x)}{\pi(a|x)}Pr(\Matrix{X}=x)\\
=&\sum_{x\in\Matrix{X}}Y(x, a)\mathbb{I}\{\pi(x)=a\}Pr(\Matrix{X}=x)\\
=&\mathbf{E}_{\Matrix{X}}[Y(\Matrix{X},\pi(\Matrix{X}))|\pi(\mathbf{X})=a]
\end{aligned}
\end{equation*}
\end{proof}

The proof follows directly the one for Inverse Propensity Score \cite{rosenbaum1983central} except dealing with each action separately through importance sampling. 
Then we have
$$\sum_{a=0}^{K}\mathbf{E}_{\Matrix{X}}[Z^{T}(a,\Matrix{X})|\pi(\Matrix{X}=a)] = \mathbf{E}_{\Matrix{X}}[Y(\Matrix{X},\pi(\Matrix{X}))] $$

Specifically, the nature response of customers with respect to an evaluated policy $\pi$ can also be estimated similarly.
It is worthy to notice that although the total nature response of all customers is a constant (independent to the policy), we still need to estimate it accurately to complete our evaluation for the expected uplift response.
And detailed analysis for its necessity is provided at the end of this section.
\begin{lemma} 
    Given a policy $\pi$, for each action $a\in \{0, 1,\dots,K\}$, define a new random variable
    $$ Z^{C}(\Matrix{X})=\frac{Y(\Matrix{X},0)\mathbb{I}\{p(\Matrix{X})=0\}}{p(0|\Matrix{X})} $$
    Then
    $$\mathbf{E}_{\Matrix{X}} [Z^{C}(\Matrix{X})|\pi(\Matrix{X})=a]=\mathbf{E}_{\Matrix{X}}[B(\Matrix{X})|\pi(\Matrix{X})=a] $$
    and 
    $$\mathbf{E}_{\Matrix{X}}[Z^{C}(\Matrix{X})]=\mathbf{E}_{\Matrix{X}}[B(\Matrix{X})] $$
\end{lemma}
The proof for Lemma 2 is similar to Lemma 1 except using the fact that $Y(x,0) = B(x)$ for $\forall x$.
Note that here we specifically mention the conditional expected nature response with respect to each action $a$ separately, because each of them are indeed dependent on the policy $\pi$.
Thus calculating the sample average of its realizations for corresponding samples will help determine rewards in our following design of policy gradient approach as action-based baselines (see Eqn.\ref{equ:qvalue}). 

Now we get Theorem \ref{thm:UMG}. 
Intuitively, the action response in collected data can be regarded as the real response we will observe if we use a new policy to choose actions, after correcting the shift in action proportions between the old data collection policy and the new policy.
Then the expected action responses for different actions and the expected nature responses can be estimated separately. 
And the difference between them is desired uplift response.

\begin{theorem}\label{thm:UMG}
Given a policy $\pi$, the expected uplift response under $\pi$ is
\begin{equation*}
\begin{aligned}
\mathbf{E}_{\Matrix{X}}[L(\Matrix{X},\pi(\Matrix{X})] =\mathbf{E}_{\Matrix{X}}[Y(\Matrix{X},\pi(\Matrix{X}))] - \mathbf{E}_{\Matrix{X}}[B(\Matrix{X})]
\end{aligned}
\end{equation*}
Let $x_i$ be the feature of $i$-th customer and $N$ be the number of customers, then the difference between sample average of $Z^T(a,\Matrix{X})$ and $Z^{C}(\Matrix{X})$
$$ \overline{z}(\pi) = \frac{1}{N}\sum_{i=1}^{N}\sum_{a=0}^{K}z^T(a,x_i)\mathbb{I}\{\pi(x_i)=a\} - \frac{1}{N} \sum_{i=1}^{N} z^C(x_i)$$ is an unbiased estimate of $\mathbf{E}_{\Matrix{X},\pi(\Matrix{X})}[L(\Matrix{X},\pi(\Matrix{X}))]$.
\end{theorem}

We call this unbiased estimator $\overline{z}(\pi)$ \textsl{Uplift Modeling General Metric} (UMG).
According to Theorem 1 and Chebyshev's inequality, suppose the variance of the UMG metric $\overline{z}(\pi)$ is bounded by $\sigma^2$, then with probability at least $1-\epsilon$,
$$|\overline{z}(\pi) - \mathbf{E}_{\Matrix{X},\pi(\Matrix{X})}[L(\Matrix{X},\pi(\Matrix{X}))]| \leq \frac{\sigma}{N\sqrt{\epsilon}}$$
which is an upper bound for UMG's sample complexity.

Such an unbiased estimator can be performance metric for any uplift modeling with multiple actions and general response types.
In other words, our objective is actually to find a policy with maximal UMG when applied to any specific dataset.

\subsection{Self-Normalized Uplift Modeling General Metric} 
Suggested by \cite{swaminathan2015self}, \textsl{Self-Normalization Estimator} is commonly used to control variance for estimators through importance sampling.
Specific to our metric, we can adjust UMG to
$$\hat{z}^{SN}(\pi) = \frac{\sum_{i=1}^{N}\sum_{a=0}^{K}z^T(a,x_i)\mathbb{I}\{\pi(x_i)=a\}}{\sum_{i=1}^{N}\frac{\mathbb{I}\{\pi(x_i)=a_i\}}{p(a_i|x_i)}} - \frac{\sum_{i=1}^{N}z^C(x_i)}{\sum_{i=1}^{N}\frac{\mathbb{I}\{a_i=0\}}{p(0|x_i)}}$$
Here $x_i$ is the $i$-th individual's feature in the dataset, while $a_i$ is the corresponding action taken on him.
The idea of introducing \textsl{Self-Normalization Estimator} is to use the \textsl{standardized weights} modify the difference between the sample average and expected value of importance sampling weights.
Based on the standard theory on ratio estimates, the bias is of order $1/N$ and can be ignored for large $N$, i.e. it is asymptotic unbiased.
And its variance is reduced (See \cite{kong1992note} for a more refined analysis.)
Thus, we call this metric \textsl{Self-Normalized Uplift Modeling General Metric} (SN-UMG).

Both UMG and SN-UMG can be evaluation metric to measure approaches for uplift modeling, and we compare their efficiency by simulation experiments in Section \ref{subsec:metric_compare}.
At the same time, SN-UMG is further used to help estimate the rewards and corresponding Q-value in the training process of our approach (see Alg.1). 
Proper reward design is one of the most important factors for a successful reinforcement learning process, and such an asymptotic unbiased estimation help design exact rewards.

\subsection{Relationship between Uplift Modeling and Offline Contextual Bandit Problem}\label{diff}

Theoretically, the optimal policy $\pi^*$ that maximize the expected uplift response is also the optimal policy that maximize the expected action response, that is
\begin{eqnarray}
\begin{aligned}
\pi^* &= argmax_{\pi} \mathbf{E} _{\Matrix{X}, \pi(\bm{X})}[L(\Matrix{X},\pi(\Matrix{X}))] \\
&= argmax_{\pi} \mathbf{E} _{\Matrix{X}, \pi(\bm{X})}[Y(\Matrix{X},\pi(\Matrix{X})) - B(\Matrix{X})] \\
&= argmax_{\pi} \mathbf{E} _{\Matrix{X}, \pi(\bm{X})}[Y(\Matrix{X},\pi(\Matrix{X}))]
\end{aligned}
\end{eqnarray}
There are also problems with objective of maximizing the expected action response.
For example, the \textsl{offline contextual bandit problem}\cite{beygelzimer2009offset} seeks the optimal policy of taking actions on individuals according to the dataset from a previous experiment.
Similar to uplift modeling, for each individual (feature), only a specific action was taken, knowing the corresponding response.
In terms of the optimal solution, the offline contextual bandit problem could be equivalent to uplift modeling by artificially dealing with data related to no-action as one ordinary action.
However, since it is never possible to obtain the optimal policy in closed form, these two objectives are indeed different when seeking the approximated optimal policies.

Before we show their difference, we first define the performance of an approximated algorithm.
Suppose the approximated algorithm $\mathcal{A}$ take dataset $D$ as input and $\pi_{\mathcal{A},D}$ as output.
The objective of it is to maximize $Obj(\pi_{\mathcal{A},D})$ and the optimal solution is denoted by $\pi^*$.
Then the performance of the algorithm is naturally defined by to what tend it approaches the optimal solution in the sense of objective, i.e.
$$\xi_{\mathcal{A},D}=\frac{Obj(\pi_{\mathcal{A},D})}{Obj(\pi^*_{D})}$$

Now we focus on the class of approximate algorithm $\mathcal{A}$ which takes a dataset of structure similar to uplift modeling, i.e. $D_r = \{(x,a,r(x,a),p(a|x))|x\in \Matrix{X}\}$, as input and $\mathbf{E}_{\Matrix{X},\pi_{\mathcal{A}}}[r(\Matrix{X},\pi_{\mathcal{A}}(\Matrix{X}))]$ as objective.
And suppose that we are able to obtain data for both action responses and uplift responses, $D_{Y}$ with $r(x,a)=Y(x,a)$ for $\forall x,a$ and $D_{L}$ with $r(x,a)=L(x,a)$ for $\forall x,a$ as an ideal case.
Then the algorithm $\mathcal{A}$'s performance should be similar when taking $D_{L}$ and $D_{Y}$ as input and corresponding objective separately, i.e. $\xi_{\mathcal{A},D_Y} \approx \xi_{\mathcal{A},D_L}$, and smaller than one of course.

On the other hand, when we evaluate the output policy of the former one $\pi_{\mathcal{A},D_Y}$ by the objective of uplift response, we have
\begin{equation*}
\begin{aligned}
&\frac{\mathbf{E}_{\Matrix{X},\pi_{\mathcal{A},D_Y}} [L(\Matrix{X},\pi_{\mathcal{A},D_Y}(\Matrix{X}))]}{\mathbf{E}_{\Matrix{X},\pi^*_{D_L}} [L(\Matrix{X},\pi^*(\Matrix{X}))]}\\
\leq&\frac{\mathbf{E}_{\Matrix{X},\pi_{\mathcal{A},D_Y}} [L(\Matrix{X},\pi_{\mathcal{A},D_Y}(\Matrix{X}))]+ \mathbf{E}_{\Matrix{X}} [B(\Matrix{X})]}{\mathbf{E}_{\Matrix{X},\pi^*_{D_L}} [L(\Matrix{X},\pi^*(\Matrix{X}))]+ \mathbf{E}_{\Matrix{X}} [B(\Matrix{X})]}\\
=&\frac{\mathbf{E}_{\Matrix{X},\pi_{\mathcal{A},D_Y}} [Y(\Matrix{X},\pi_{\mathcal{A},D_Y}(\Matrix{X}))]}{\mathbf{E}_{\Matrix{X},\pi^*_{D_Y}} [Y(\Matrix{X},\pi^*(\Matrix{X}))]} = \xi_{\mathcal{A},D_Y} 
\end{aligned}
\end{equation*}
Here the first equality in the last line holds since the optimal solution is the same for these two problem, $\pi^*_{D_L} = \pi^*_{D_Y}$, as we see before.
And whether the inequality in the second line holds depends on the expected nature response $\mathbf{E}_{\Matrix{X}} [B(\Matrix{X})]$.
We conclude the analysis formally in the following theorem.

\begin{theorem}\label{thm:upliftbetter}
If $\mathbf{E}_{\Matrix{X}} [B(\Matrix{X})] = 0$, then algorithm $\mathcal{A}$, taking input $D_{Y}$, has the same performance on the objective of expected uplift response and on the one of expected action response.

If $\mathbf{E}_{\Matrix{X}} [B(\Matrix{X})] > 0$, then algorithm $\mathcal{A}$, taken input $D_{Y}$, will always get worse performance on the objective of expected uplift response, compared with its performance on the objective of expected action response.
\end{theorem}
In other words, if the expect nature response is zero, uplift modeling and offline contextual bandit problem are equivalent.
Otherwise, as in most common cases, the expected nature response is positive \footnote{For example, there are always some customers coming to a restaurant even no discount any for food is provided.}, these two problems are inherent different in the sense of approximation solutions.
And since we are considering uplift modeling problem, it is better to seek a solution directly related to the uplift response, rather than solving another problem which seems to be equivalent but turns out to be not.

\section{Reinforcement Learning Method For Uplift Modeling}\label{sec:methods}

\subsection{Overview}
In this section, we first show how to reformulate the uplift modeling problem as an MDP problem by constructing an equivalent Markov chain for the problem. Then we show how to use policy gradient algorithm for solving uplift modeling problem in detail.
The uplift modeling problem is particularly suitable for MDP and RL reformulation, since the exact uplift value is typically not provided for each individual sample, but we can estimate the average value for a batch experiment statistically according to Theorem \ref{thm:UMG}.
Such a situation corresponds to receiving the delayed reward signals after an entire episode of MDP.

In summary, an overview of our approach for uplift modeling is illustrated in Fig. \ref{fig:transition}(a).
It contains two parts, the actions selection and evaluation.
The policy network chooses an action for each sample, then the output will be evaluated by the evaluation function. The policy network updates its parameter according to the result of evaluation iteratively.

\begin{figure*}[ht]
    \center
    \subfigure[Overall framework.]{
        \label{subfig:framework}
        \includegraphics[width=3.6in]{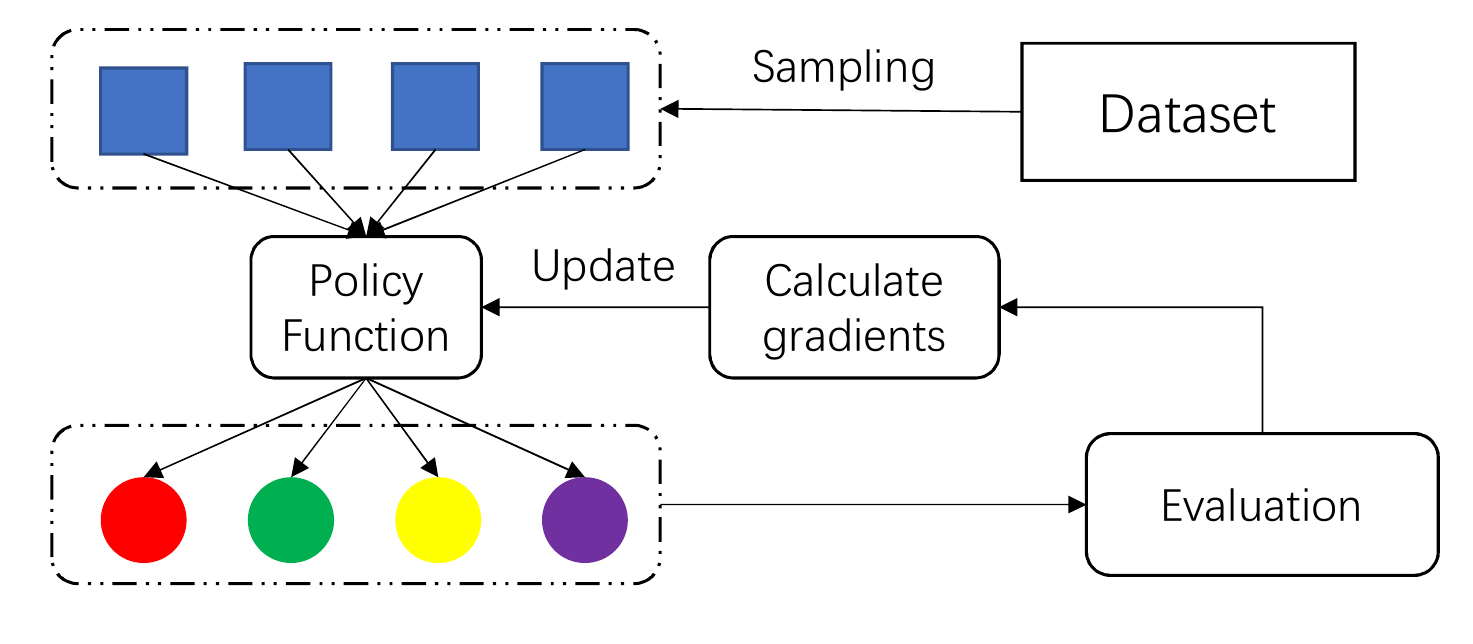}
    }
    \subfigure[The MDP model of uplift modeling.]{
        \label{subfig:transition}
        \includegraphics[height=1.5in]{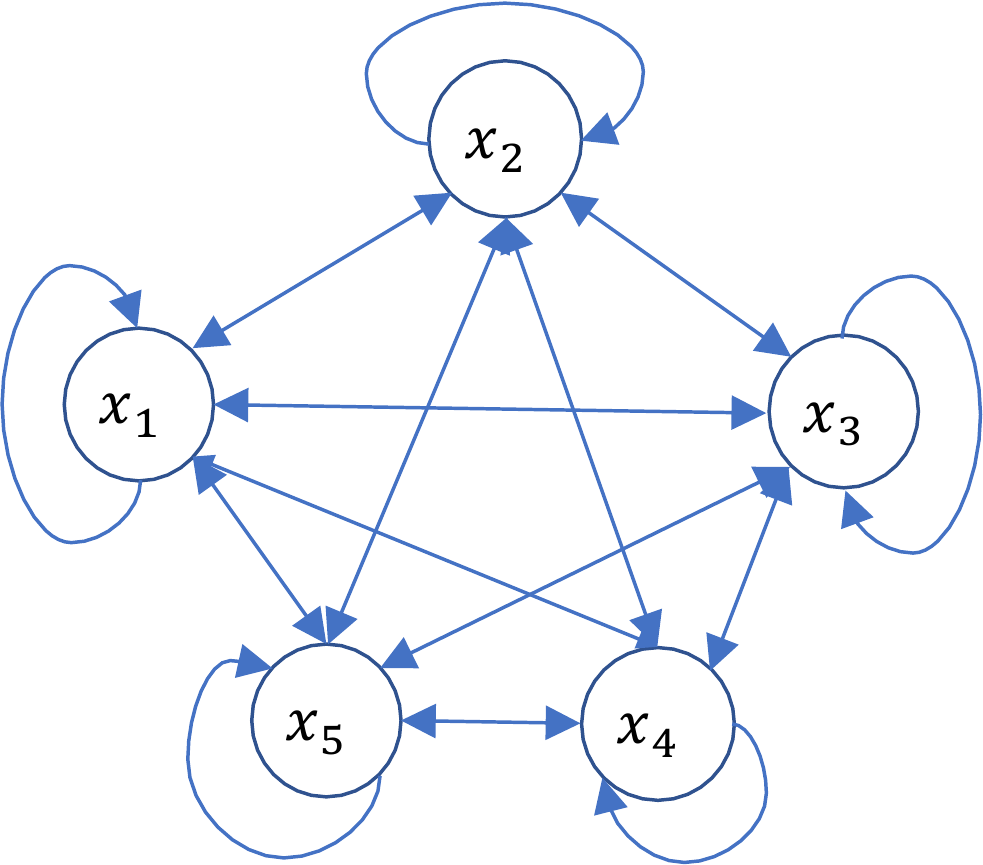}
    }
    \caption{(a) The overall framework. The blue squares represent the samples from dataset randomly. The colorful circles are the actions, coming from the policy. The results are evaluated by the evaluation function. Then the policy are updated according to evaluation results. (b) An example of the MDP model of uplift modeling with only five states.}
    \label{fig:transition}
\end{figure*}

\subsection{MDP Model for Uplift Modeling}
A Markov decision process is 4-tuple, ($S, A, P, R$), where
$S$ is the set of states.
$A$ is the set of actions.
$P_{ss'}^a = Pr \{ s_{t+1} = s' | s_t = s, a_t = a \} $ is the probability that if we choose action $a$ at time $t$ in state $s$, it will lead to $s'$ at time $t+1$, for all $s \in S$.
$R_{s}^{a} = E \{ r_{t+1} | s_t = s, a_t = a \}$ is the immediate reward at the time $t$, transiting from state $s$ to state $s'$, due to action $a$.

We can model the uplift problem as an MDP.
Recall our definition of uplift modeling, Eqn.(\ref{equ:def_uplift}), it aims to maximize the average value of the uplift response.
Thus, at each time $t$, the agent observes a state $s_t$ (i.e, user's feature $x$), and its policy chooses an action $a_t$ (i.e one action or no action), and gets the reward $R_s^a = L(x,a)$ (the uplift response for the user after receiving the action).
The transition probability from $s$ to $s'$ is independent on the state and action. $P_{ss'}^a = 1 / |S|$ are equal for all the pairs ($s$, $s'$), $s, s' \in S$.
Formally, we define $S = X$, $A = \{ 0, 1, ..., K \} $ if we have $K$ actions.
$P_{ss'}^a = 1, \forall s \in S$.
In the uplift modeling, $P_{ss'}^a = Pr \{ s_{t+1} = s' | s_t = s, a_t = a \}$ is always the same, and the stationary distribution of this constructed Markov chain is a uniformly random distribution over all samples in the dataset. Thus sampling from the Markov chain is equivalent to sampling uniformly from the dataset, and independent on the action $a_t$. The Fig. \ref{fig:transition}(b) shows the MDP of uplift modeling.\\
\textbf{State}:
State in MDP is object's features $x$ in uplift modeling.
\textbf{Action}:
MDP has the same action set as original uplift modeling, $A = \{ 0, 1, \dots, K \} $. The policy $\pi(\cdot)$ decides which action to choose to maximize the reward.
We sample the action $a_i \sim \pi(s_i)$. We adopt a softmax function as the policy function.\\
\textbf{Transition}:
Each time $t$ the state $s_t = x_i$ may transit to any state $s_{t+1} = x_j$ with equal probability.

\subsection{Policy Gradient Method}
Reinforcement learning agents learn to maximize their expected future rewards from interaction with an environment.
Each time, we can choose an action $a$ when the current state is $s$ according to a policy $\pi(\theta)$, $a \sim \pi(s, \theta)$, where $\theta$ is parameter of the neural network.
We can evaluate the policy $\pi(\theta)$ according to their long-term expected reward,
\begin{equation}
    J(\theta) =
    \mathbf{E}_{\Matrix{X}, \pi(\Matrix{X})}
    [L(\Matrix{X}, \pi(\Matrix{X}))]
    \label{equ:average_reward}
\end{equation}
Therefore, we can directly optimize the parameter $\theta$ to maximize $J$.
According to the policy gradient theorem \cite{sutton2000policy}, we can calculate the gradient by:
\begin{equation}
    \nabla_{\theta} J(\theta) = \mathbf{E}_{\pi(\theta)}[\nabla_{\theta} \mathrm{log} (\pi(a|s, \theta) Q^{\pi(\theta)}(s, a))]
    \label{equ:policy_gradient},
\end{equation}
where $Q^{\pi}(s, a)$ is state-action value given a policy $\pi$ and is defined as
\begin{equation}
    Q^{\pi}(s, a) = R_s^a + \sum_{s'} P_{ss'}^a V^{\pi}(s')
\end{equation}
where $s'$ is the next state and $V^{\pi}(s')$ is state-value started with $s'$, defined as
\begin{equation}
\begin{aligned}
    V^{\pi}(s) &= \lim\limits_{n \rightarrow \infty} E[\frac{R^{\pi(s_1)}_{s_1} + ... + R^{\pi(s_n)}_{s_n}}{n}] \\ 
    &= \mathbf{E}_t [L(s_t, \pi(s_t))| s_0 = s], \forall s \in S.
\end{aligned}
\end{equation}
We denote $V^{\pi}(s^*) = \sum_{s'} P_{ss'}^a V^{\pi}(s')$ for convenience.
According to Eqn. (\ref{equ:policy_gradient}), the key for calculating the gradient is to know $Q^{\pi}$. Thus we introduce how to estimate $Q^{\pi}$ specific to the uplift modeling as follows:

\paragraph{Q-value Estimation}
$Q^{\pi}(s, a)$ contains two part, $R_s^a$ and $V^{\pi}(s^*)$.
The $V^{\pi}(s^*)$ can be estimated by calculating the SN-UMG.
In each episode, we randomly sample $M$ batches of samples according the MDP, $\mathbf{\Gamma} = \{ \Gamma_1, ..., \Gamma_M \}$, to estimate the value.
$L(s, a)$ is hard to know, so we expect $R_{s}^{a}$ to approximate it.  
According to Theorem \ref{thm:UMG}, we can find: (1) If $p(s)=\pi(s)=a$, then the response $Y(s, a)$ has the positive impact on the result of action $p(s)$. (2)  If $p(s) = 0$ and $\pi(s)=a > 0$, then $Y(s, 0)$ has the negative impact on the result of action $p(s)$, (3) For other cases, $Y(s, a)$ has no impact on the result, so that we do not consider these cases.
Thus, we set $R_{s}^{a}$ as the estimation for the uplift response of this specific sample according to UMG metric
\begin{equation}
    R_{s}^a = 
    \left\{
    \begin{array}{lr}
    Y(s, a)\frac{1}{p(a|s)}, ~~~~ \mathrm{if}~ p(s)=\pi(s)=a \\
    -Y(s, 0)\frac{1}{p(0|s)}, ~ \mathrm{if}~ p(s) = 0  ~\mathrm{and}~ \pi(s)=a > 0
    \end{array}
    \right.
\end{equation}

In addition, in order to accurately estimate the value $V^{\pi}$, the size of each batch need to be large. We need to subtract the baselines to reduce the variance. We also adopt the action-dependent baseline to reduce the variance, which is shown to be effective recently, that is
\begin{equation}\label{equ:qvalue}
Q^{\pi}(s, a) =  \left\{
\begin{array}{lr}
(Y(s, a)- \overline{Y^C}(a)) / p(a|s)  + (V^{\pi}(s^*) - \overline{V^{\pi}}(s^*))  \\
~~~~~~~~~~~~~~~~~~~~~~~~~~~~~~\mathrm{if}~ p(s)=\pi(s)=a &\\
(\overline{Y^C}(a) - Y(s, 0)) / p(0|s)+ (V^{\pi}(s^*) - \overline{V^{\pi}}(s^*)) \\
~~~~~~~~~~~~~~~~~~~~~~~~~~~~~~\mathrm{if}~ p(s) = 0  ~\mathrm{and}~ \pi(s)=a > 0 &
\end{array}
\right.
\end{equation}
Here $\overline{Y^C}(a) = \sum_{m=1}^{M} Y^C_m(a) / M$ and $Y^C(a)$ is the sample average of the conditional expectation for $Z^C(\Matrix{X})$ with respect to action $a$ as we introduced in Lemma 2, which can be calculated in the process of SN-UMG.
$\overline{V^{\pi}}(s^*) = \sum_{m=1}^{M} V^{\pi}_m(s^*) / M $ is the average value of multiple batches in order to estimate the $V^{\pi}(s^*)$ accurately.
Finally, we optimize the $\theta$ for each batch $\Gamma_m$,
\begin{equation}
    \theta \leftarrow \theta + \alpha \sum_{i=1}^{|\Gamma_m|} \nabla_{\theta} \mathrm{log} \pi(s_i, a_i) Q^{\pi}(s_i, a_i), s_i \in \Gamma_m.
\end{equation}

The whole algorithm is shown in Algorithm. \ref{alg:rlift}.
\begin{algorithm}[ht]
                \caption{\small{Policy Gradient Approach for Uplift Modeling}}
                \KwIn{Episode number $numEpoch$. Training data $Data$, batch size $bs$, learning rate $\alpha$}
                \KwOut{The policy network $\theta$}
                \For{$epoch$ $\leftarrow 1$ \KwTo $numEpoch$}{
                    Sample $M$ batches $\mathbf{\Gamma} = \{ \Gamma_1, \dots, \Gamma_M\}$ from $Data$, where each batch contains $bs$ samples. \\
                    \ForEach{$\Gamma_m \in \mathbf{\Gamma}$}{
                        $A_m = \{ a_{m, 1}, \dots, a_{m, bs} \}$,  where $a_{m, i} \sim \pi(s_{m, i}, \theta)$ \\
                        $V_m^{\pi}(s^*), \{ \overline{Y^C(a)} \}_{a=0}^{K} = \mathrm{SN-UMG}(\Gamma_m, A_m)$\\
                    }
                    $\overline{V^{\pi}}(s^*) = \sum_{i=1}^{M} V_m^{\pi}(s^*) / M$ \\
                    \For{$m$ $\leftarrow 1$ \KwTo $M$}{
                        Compute the $Q^{\pi}(s_{m, i}, a), \forall s_{m, i} \in \Gamma_m$, according to Eqn.\ref{equ:qvalue}\\
                        $\theta \leftarrow \theta + \alpha \sum_{i=1}^{bs} \nabla_{\theta} \mathrm{log} \pi(s_{m, i}, a_{m, i}) Q^{\pi}(s_{m, i}, a_{m, i})$
                    }
                }
                \label{alg:rlift}
\end{algorithm}

\section{Experiments}\label{sec:experiemtns}
\subsection{Experiment Setup}\label{experiment:setup}
\subsubsection{Dataset}
We first introduce the open dataset, simulation dataset and real business dataset we used to evaluate our method compared with other baselines.

\paragraph{MineThatData} Kevin Hillstrom’s MineThatData blog \cite{KevinData} is an open dataset containing results of an e-mail campaign for an Internet-based retailer.
It contains information about 64,000 customers with basic marketing information such as the amount of money spent in the previous year or when the last purchase was made.
We use the part of the dataset which containing the visiting response and women's advertisement, because the men' advertisements are ineffective and purchasing signals are sparse.
We use it for the single action and binary response experiment.

\paragraph{Synthetic Dataset}\label{synthetic}
For the multiple actions and general response type experiments, there is no open dataset large enough, thus we generate a simulation dataset.
The generation algorithm is a modified version in \cite{zhao2017uplift}.
\cite{zhao2017uplift} proposed a method based on the decision tree, so the uplift value of different actions in their dataset depends on only one attribute, while our method has no such requirement.

The features space is a 50-dimensional hyper-cube of length 10, or $\mathbb{X}^{50}$ = $\mathrm{[} 0, 10 \mathrm{]}^{50}$.
Features are uniformly distributed in the feature space, i.e., $X_d \leftarrow \mathrm{U} \mathrm{[} 0, 10 \mathrm{]}$, for $d = 1, \dots, 50$.
There are four     different actions, $A = \{ 1, 2, 3, 4 \}$.
The response under each action is defined as below.
\begin{equation}
    \begin{split}
        Y(X) &= 5f_1(X) + f_2(X) + \epsilon \\
        f(x_1, \dots, x_{50}) &= \sum\limits_{i=1}^{50} a^i \cdot \mathrm{exp} \{-\sum\limits_{j=1}^{50}b_j^i|x_j - c_j^i|\}.
    \end{split}
\end{equation}
The action response $Y(X)$ consists of the nature response ($f_1$), the uplift response ($f_2$) and white noise ($\epsilon$).
Both $f_1(\cdot)$ and $f_2(\cdot)$ are in the form of $f(\cdot)$, but with different parameters ($a, b, c$).
In our experiment, we set $a^i \sim \mathrm{U}\mathrm{[}0, 10\mathrm{]}$, $b_j^i \sim \mathrm{U}\mathrm{[}0, 0.1\mathrm{]}$ and $c_j^i \sim \mathrm{U}\mathrm{[}0, 5\mathrm{]}$ for $\forall i,j$.
And a group of $a^i$, $b_j^i$ and $c_j^i$ for $\forall i,j=1,2,\dots,50$ is randomly chosen for the nature response $f_1(\cdot)$.
Then for each action, a new group of $a^i$, $b_j^i$ and $c_j^i$ for $\forall i,j=1,2,\dots,50$ is randomly chosen for the corresponding uplift response $f_2(\cdot)$.
Finally $\epsilon$ is set to be the zero-mean Gaussian noise, i.e. $\epsilon \sim N(0, \sigma^2)$.
We set $\sigma = 0.8$.
We generate 500,000 samples for each action and a control group ($f_2(\cdot)=0$) with 500,000 samples.
Thus, there are 2500,000 samples in total.

\paragraph{Real Business Dataset}
We also evaluate our methods on the real business dataset from a company.
The dataset is selected from its marketing records for its new service in September 2017, when coupons of different types were sent to customers to attract them to use the service and further become long-term memberships.
The type of each coupon was randomly chosen with equal probabilities for all levels, independent of customers' features.
We took 620,000 samples of these customers' features (264 related attributes, such as one's resident, age, gender and so on), types of their received coupons (actions), and their response (whether they used the service, denoted by $r_1$, and whether they pay for long-term memberships, denoted by $r_2$). The ratio of the positive and negative sample is close 1 : 200.

\subsubsection{Baselines}
We compare our methods with several baselines on both single action and multiple actions.
\begin{itemize}
    \item \textbf{Separate Models Approach (SMA)} \cite{zhao2017uplift} Using a separate model for each group of people receiving the same action and predicting the response given each specific action and features. Choose the actions with the largest response. It can be applied to multiple actions and general responses. We consider both Random Forest and neural network as the separate models, denoted by SMA-RF and SMA-NN.
    \item \textbf{Random Forests (Uplift RF)} \cite{guelman2015uplift} Using a specific function as the splitting criteria for random forests. We use the package implemented in R, which can only be applied on single action and binary response.
    \item \textbf{Offset Tree (OT)} \cite{beygelzimer2009offset} Reducing the problem into binary classification and reusing fully supervised binary classification algorithms (also known as base algorithm). We use the Python package {\href{https://github.com/david-cortes/contextualbandits}{\textsl{contextualbandits}}} which requires more than two actions. And we consider both Random Forest and Logistic Regression as the base algorithms, denoted by OT-RF and OT-LR.
\end{itemize}

\subsubsection{Parameter Selection}
We evaluate our model by using three-fold cross validation.
The neural network of the policy gradient is three-layers fully connected network with the hidden layer being a size of 512 for the open dataset and synthetic dataset, and a size of 1024 for the real-world business data.
Each episode, we take 10 random batches and each batch contains 10000 examples.
The splitting ratio of training, validating and testing is 0.6, 0.2, 0.2.
We stop the training when the model could not achieve a better result on the validate dataset within 1000 episodes. The learning rate is 0.1.

\subsection{Efficiency of UMG and SN-UMG}
In this subsection, we verify our proposed metric UMG and SN-UMG on the  synthetic datasets and show the convergence of our two proposed metrics.

The dataset for experiments is generated by the method we introduce in the Section \ref{synthetic}. The action set is also $\{0, 1, 2, 3, 4 \}$.
We adopt two kinds of policies for action offering in two experiments:
\begin{itemize}
    \item Uniform: Each action is offered with the equal probability. That is, $Pr(X, a) = 0.2, \forall X, a$.
    \item Policy: We choose first five attributes of each X, $x_1, \dots, x_5$. The probability of choosing the action $i$ is proportional to $x_i$. That is,  $Pr(X, a) = \frac{x_a}{\sum_{i=1}^{5}x_i}$
\end{itemize}
We test the errors between the real uplift response and the estimations of uplift response from UMG and SN-UMG with different sizes of data, in order to test its convergence efficiency.
When the size of data is fixed, we run 10 times of experiments to estimate the mean and variance of two metrics, which are shown in Fig. \ref{fig:random:umg} and Fig. \ref{fig:policy:umg}.
The Fig. \ref{fig:random:umg} shows the convergence curve of UMG and SN-UMG on the uniform dataset.
The Fig. \ref{fig:policy:umg} The convergence curve of UMG and SN-UMG on the policy dataset.
We can find that SN-UMG performs better in both experiments than UMG, on the accuracy and stability.
When the size of data is over 10000, both metrics almost converge.
Both metrics perform better on the uniform data than on the policy data, while the variances of SN-UMG are relatively lower than the ones of UMG.
\begin{figure}[ht]
    \centering
    \includegraphics[width=0.95\linewidth]{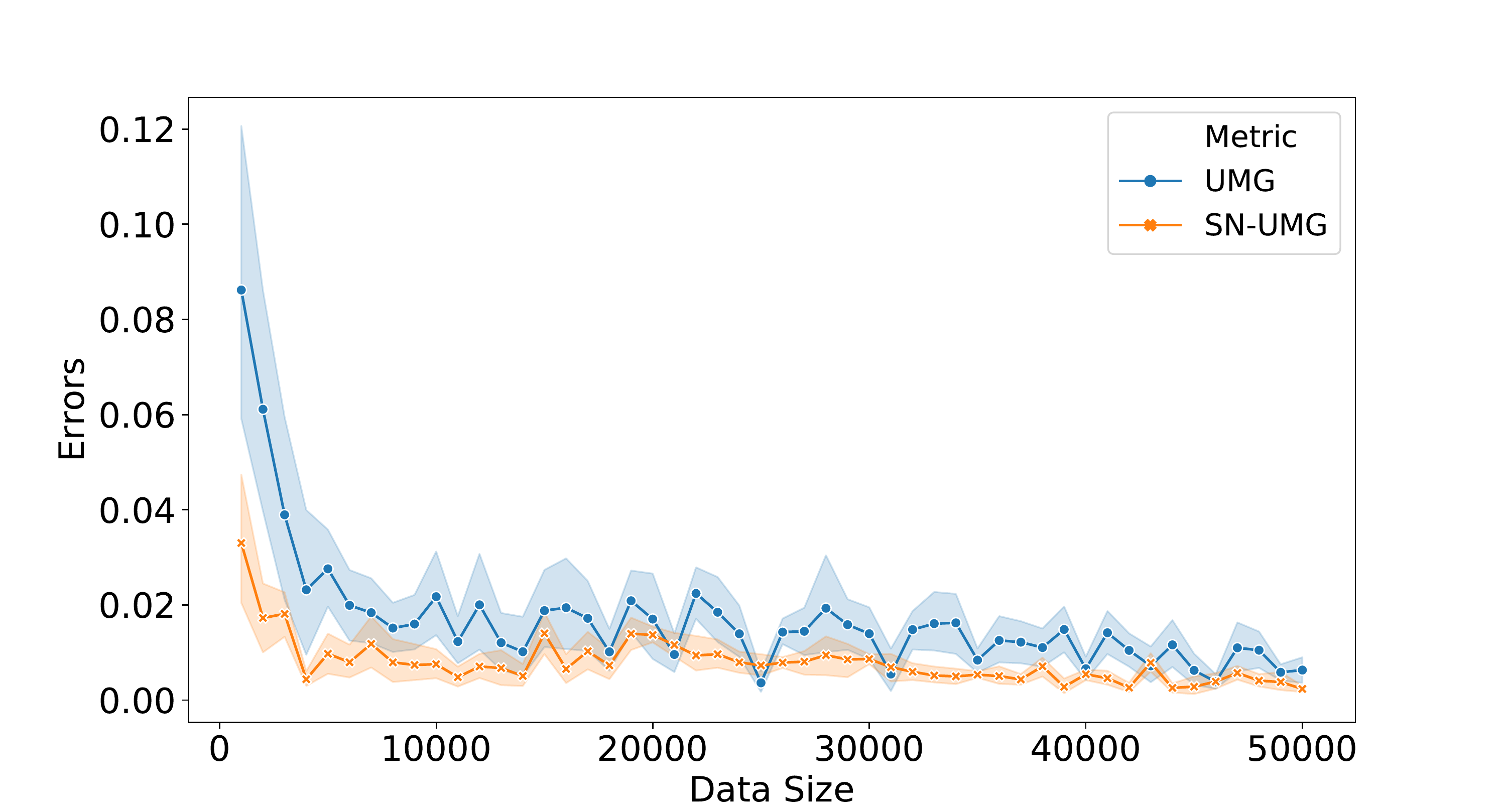}
    \caption{The convergence curve of UMG and SN-UMG on the uniform dataset.}
    \label{fig:random:umg}
\end{figure}

\begin{figure}[ht]
    \centering
    \includegraphics[width=0.95\linewidth]{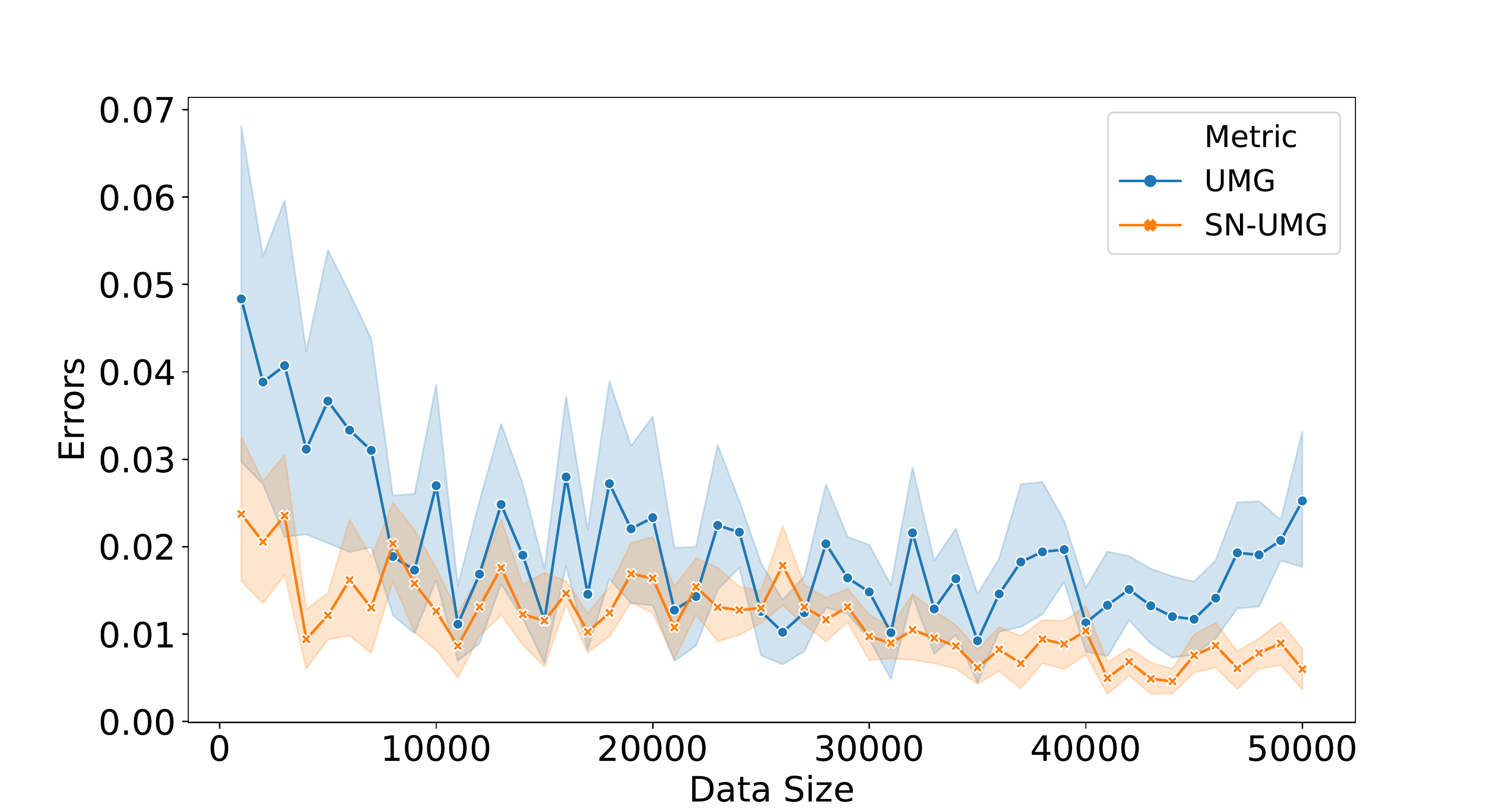}
    \caption{The convergence curve of UMG and SN-UMG on the policy dataset.}
    \label{fig:policy:umg}
\end{figure}

\subsection{Single Action and Binary Response}\label{subsec:metric_compare}
We compare our method with SMA-RF, SMA-NN and Uplift RF. The dataset is MineThatData dataset.
Methods are evaluated not only on our proposed UMG metric, but also on the Qini curve and Qini coefficient\cite{radcliffe2007using}.
Basically, the larger the area of one policy's Qini curve is, the better its performance is.
And Qini coefficient is calculated by this area first subtracted by the area of a random policy, then divided by a constant only related to the dataset.

%
%
And evaluating one policy by Qini curve requires the probability that it chooses to take the single action $a^*$, instead of its binary output (take $a^*$ or not), for each individual, thus we adjust our method's output when evaluated by Qini.
In precise, we adopt multiple discrete actions ($\{0,\dots,n-1\}$) to represent different probabilities of offering the single action in this experiment.
For example, when our policy chooses action $a$ ($0\leq a \leq n-1$) for one sample, it means the probability of taking action $a^*$ for this sample is between $\frac{a}{n}$ and $\frac{a+1}{n}$.
In our experiment, $n=5$.

The results are shown in Table \ref{tab_singlemulti} and we also plot the Qini curves for methods with large Qini coefficients in Fig. \ref{fig_qini}.
Uplift RF performs better than SMA-RF significantly indicates that the method designed for uplift modeling is more effective on equal conditions.
SMA-NN gains the best results except for RLift manifests that it's necessary to import more powerful model, like neural network, to the uplift modeling.
Our method can achieve the state-of-the-art performance on both the SN-UMG metric and Qini coefficient, because we model the uplift signal and adopt deep learning method simultaneously.

\begin{figure}[ht]
    \centering
    \includegraphics[width=\linewidth]{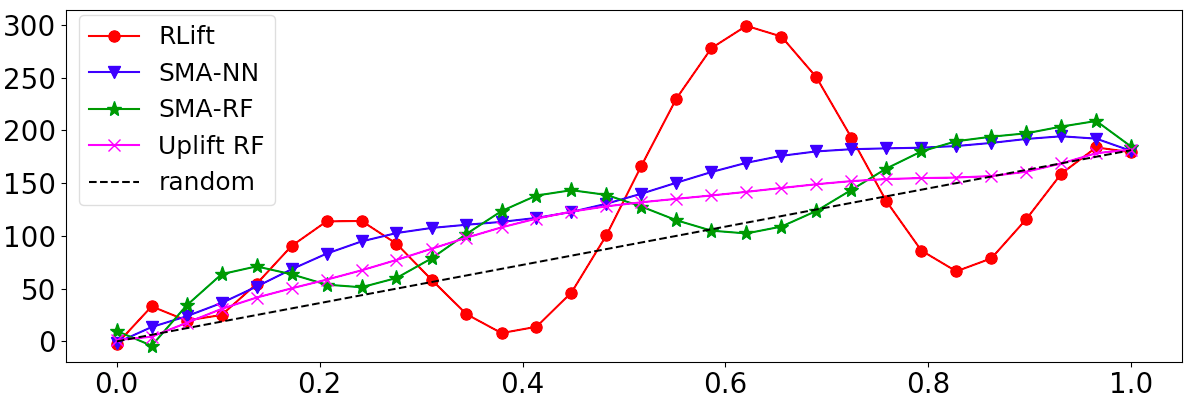}
    \caption{Qini Curve}
    \label{fig_qini}
\end{figure}

\begin{table}[ht]
    \centering
    \begin{tabular}{|c|c|c|c|c|c|c|}
    \hline
    Methods & SMA-RF & SMA-NN & \multicolumn{2}{|c|}{Uplift RF} & \multicolumn{2}{|c|}{RLift}\\ \hline
    SN-UMG  & 0.0223 & 0.0364 & \multicolumn{2}{|c|}{0.0313} & \multicolumn{2}{|c|}{\textbf{0.0464}} \\ \hline
    Qini & 0.0246 & 0.0650 & \multicolumn{2}{|c|}{0.0412} & \multicolumn{2}{|c|}{\textbf{0.0656}} \\ \hline
    \end{tabular}
    \caption{\small{Comparison of \textit{SN-UMG} and \textit{Qini} on the MineThatData}}
    \label{tab_singlemulti}
\end{table}

It is worth noticing that fluctuations occur in Qini curves, for RLift, SMA-RF and SMA-NN.
This is because these methods are optimized based on the whole dataset of samples, while the Qini curve shows some intermediate results for individual samples.
Since the final objective is measured by the area between its Qini curve and the horizontal axis, we conjecture there may be a trade-off between the final results and intermediate result.
On the other hand, Qini curve can be only applied on the case of binary response and binary action while our method is not limited to it.
Thus it is not suitable for further experiments and we just show the relationship between our method and previous works in this experiment.

\subsection{Multiple actions and General Response}
For uplift modeling with multiple actions, only Uplift KNN, SMA-RF, and SMA-NN can be applied as baselines.
For the dataset, there are not large enough open dataset, so we generate the synthetic dataset, among which the optimal results are known.
The evaluation metric is SN-UMG.

The results are shown in Table \ref{tab_singlemulti}.
Our method performs much better than the baselines. When the relation between responses and features are extremely complicated, the advantage of RLift is more obvious, compared with other baselines.
Also, when the size of the dataset is large enough, the performance of RLift is very close to the optimal result.

\begin{table}[ht]
    \centering
    \begin{tabular}{|c|c|c|c|}
    \hline
    Methods & SMA-RF & SMA-NN & OT-RF\\ \hline
    SN-UMG & 0.1010 & 0.1158 & 0.1075 \\ \hline
    Methods & OT-LR & Optimal & RLift \\ \hline
    SN-UMG & 0.1270 & 0.1467 & \textbf{0.1397} \\ \hline
    \end{tabular}
    \caption{Comparison of \textit{SN-UMG} on the Synthetic Data}
    \label{tab_multi}
\end{table}

\subsection{Real Business Experiments}
We also test our method on the real-world business dataset, evaluated through SN-UMG.
It contains records of multiple actions with binary response, thus we use SMA-RF and SMA-NN as baselines.
Uplift KNN is not considered due to its low efficiency when deploying in the environment of big data.

And in the dataset, there are two kinds of binary responses ($r_1,r_2 \in\{0,1\}$), related to two objectives in practice.
Thus we conduct two kinds of experiments in this part.
Firstly, we consider the uplift value of these two binary responses as the objectives separately, thus it is two experiments of multiple actions and binary response.
Table. \ref{tab:real_one} shows the results of the experiment on each objective. Since the positive sample is extremely sparse in the dataset, mentioned above, the performances of SMA-RF and SMA-NN are very bad. Random performs relatively well.
By contrast, RLift performs robustly in this task, because it uses the statistic results as rewards to guide training, which can reduce the negative effects of the sparsity of the dataset.
\begin{table}[ht]
\centering
\begin{tabular}{|c|c|c|c|c|c|}
\hline
Methods & SMA-RF & SMA-NN & OT-RF & OT-LR & RLift\\ \hline
$r_1$ & 0.00287 & 0.00329 &  0.02753 & 0.02370 & \textbf{0.03024} \\ \hline
$r_2$ & 0.000252 & 0.000793 & 0.00773 & 0.00504 & \textbf{0.00842} \\ \hline
\end{tabular}
\caption{Comparison of \textit{SN-UMG} on the Real Business Dataset}
\label{tab:real_one}
\end{table}

Secondly, we consider the weighted combination of two responses as a single objective, thus it is the experiment of multiple actions and general response.
Such an objective is common in a practical business where multiple objectives are concerned by companies simultaneously.
Table. \ref{tab:real_multi} shows the results of the experiments on the objectives with different weights.
The weight of $r_2$ is always larger than the one of $r_1$ conforms to the actual demand.
In these tasks, RLift reveals its flexible ability on multi-objective tasks and shows its adaptation to complicated tasks.
All the results show that our model can achieve the superiority compared with Random, SMA-RF and SMA-NN.
\begin{table}[ht]
\centering
\begin{tabular}{|c|c|c|c|c|c|}
\hline
Methods & SMA-RF & SMA-NN & OT-RF & OT-LR & RLift\\ \hline
(0.1, 0.9) & 0.00077 & 0.00127 & 0.00953 & 0.00563 & \textbf{0.01120} \\ \hline
(0.2, 0.8) & 0.00118 & 0.00116 & 0.01078 & 0.00722 & \textbf{0.01325} \\ \hline
(0.3, 0.7) & 0.00114 & 0.00193 & 0.01233 & 0.00905 & \textbf{0.01548} \\ \hline
(0.4, 0.6) & 0.00146 & 0.00189 & 0.01394 & 0.01091 & \textbf{0.01871} \\ \hline
\end{tabular}
\caption{Comparison of \textit{SN-UMG} on the Real Business Dataset on the weighted indexes. $(w_1,w_2)$ represent an objective of $w_1*r_1 + w_2*r_2$.}
\label{tab:real_multi}
\end{table}

\section{Conclusion}
In this paper, we propose a new evaluation metric of multiple actions and general response types for the uplift modeling and prove its unbiasedness.
And we solve the uplift modeling problem through a deep reinforcement learning method.
During the training process, the variance for estimating rewards and Q-value is further reduced by taking advantages of unbiased metric as well as action-dependent baselines.
Compared with current methods on open, synthetic and real datasets, our method achieves the state-of-the-art performance under new proposed metric as well as traditional metric.

\bibliographystyle{plain}  
\bibliography{reference}
\end{document}